\documentclass[11pt]{article}

\usepackage{fullpage}
 
\usepackage[utf8]{inputenc} 
\usepackage[T1]{fontenc}    
\usepackage{hyperref}       
\usepackage{url}            
\usepackage{booktabs}       
\usepackage{amsfonts}       
\usepackage{amssymb}
\usepackage{nicefrac}       
\usepackage{microtype}      
\usepackage{xcolor}         

\usepackage{amsmath,amsfonts}
\usepackage{natbib}
\usepackage{hyperref}
 
\usepackage{authblk}

 \usepackage[ruled,vlined,linesnumbered]{algorithm2e}
\usepackage{xcolor}
\usepackage{enumitem}
\usepackage[makeroom]{cancel}

\SetKwInput{KwInput}{Input}             
\SetKwInput{KwReturn}{Return}

\newcommand{\rE}{{\mathbb E}}

\newcommand{\rR}{{\mathbb R}}

\newcommand{\trace}{{\mathrm{trace}}}

\newcommand{\cA}{{\mathcal A}}

\newcommand{\cM}{{\mathcal M}}

\newcommand{\cZ}{{\mathcal Z}}
\newcommand{\cD}{{\mathcal D}}

\newcommand{\Acal}{\mathcal{A}}

\newcommand{\Gcal}{\mathcal{G}}
\newcommand{\Mcal}{\mathcal{M}}
\newcommand{\Zcal}{\mathcal{Z}}

\newcommand{\Kcal}{\mathcal{K}}

\newcommand{\Pcal}{\mathcal{P}}

\newcommand{\Ncal}{\mathcal{N}}
\newcommand{\Xcal}{\mathcal{X}}

\newcommand{\Dcal}{\mathcal{D}}

\newcommand{\PP}{\mathbb{P}} 
\newcommand{\RR}{\mathbb{R}} 
\newcommand{\Sfrak}{\mathfrak{S}}

\newcommand{\tp}{{\widetilde{p}}}
\newcommand{\tw}{{\widetilde{w}}}

\newcommand*\lrp[1]{\left(#1\right)}
\newcommand*\lrn[1]{\left\|#1\right\|}

\usepackage{amsthm}
\newtheorem{lemma}{Lemma}

\newtheorem{theorem}{Theorem}
\newtheorem{proposition}{Proposition}
\newtheorem{assumption}{Assumption}
\newtheorem{definition}{Definition}
\newtheorem{corollary}{Corollary}

\newcommand{\dataset}{{S}}

\title{Dimension Independent Generalization of DP-SGD for Overparameterized Smooth Convex Optimization
}

\date{}

\author{%
 {Yi-An Ma} ({yianma@ucsd.edu})\\
  University of California, San Diego
 \and
 {Teodor Marinov} ({tvmarinov@google.com})\\
  Google Research
 \and
  {Tong Zhang} \\
  Google Research and The Hong Kong University of Science and Technology
}

\begin{document}

\maketitle

\begin{abstract}

This paper considers the generalization performance of differentially private convex learning. We demonstrate that the convergence analysis of Langevin algorithms can be used to obtain new generalization bounds with differential privacy guarantees for DP-SGD. More specifically, by using some recently obtained dimension-independent convergence results for stochastic Langevin algorithms with convex objective functions, we obtain $O(n^{-1/4})$ privacy guarantees for DP-SGD with the optimal excess generalization error of $\tilde{O}(n^{-1/2})$ for certain classes of overparameterized smooth convex optimization problems. This improves previous DP-SGD results for such problems that contain explicit dimension dependencies, so that the resulting generalization bounds become unsuitable for overparameterized models used in practical applications. 
\end{abstract}

\section{Introduction}

We consider a set of i.i.d. training data $\dataset_n=\{z_1,\ldots,z_n\}$, each drawn from a distribution $\cD$ on $\cZ$. 
Consider a loss function $\ell(w,z)$ that is convex in $w$, where $w \in \rR^d$, and $z \in \cZ$. 
In supervised learning, the goal is find an estimator $\hat{w}$ from $\dataset_n$ to minimize the generalization error
\begin{equation}
    \rE_{z \sim \cD} \ell(\hat{w},z) .
    \label{eq:genloss}
\end{equation}

A frequently used method is to use stochastic gradient descent (SGD) to approximately minimize the empirical risk on the training data:
\begin{equation}
\textstyle
    \hat{w} \approx\arg\min_{w \in \rR^d} 
    \frac{1}{n} \sum_{i=1}^n \ell(w,z_i)  .
    \label{eq:regularized-erm}
\end{equation}
 It can be shown using tools from statistical learning theory that the the SGD method that aims to minimize the empirical risk in \eqref{eq:regularized-erm} 
also achieves a small generalization error in \eqref{eq:genloss}. In particular, it can be shown that under suitable conditions on the loss function $\ell(w,z)$, the excess generalization error achieved by SGD, compared to an arbitrary  $w \in \rR^d$, is bounded as
\[
\rE_{z \sim \cD} \ell(\hat{w},z) - 
\rE_{z \sim \cD} \ell(w,z) = \tilde{O}(1/\sqrt{n}) ,
\]
where the constant in $\tilde{O}(\cdot)$ depends on the norm $\|w\|_2$, properties of $\ell(w,z)$, and log factors. 

However, the standard SGD method for solving \eqref{eq:regularized-erm} does not have good privacy guarantees. To overcome this problem, there have been significant interests in designing statistical estimators that are differentially private. That is, such estimators are insensitive to a modification of a single data point in $\dataset_n$. In practice, the most widely used algorithm is DP-SGD, which adds a Gaussian noise in each SGD step. 
Mathematically this method is equivalent to the stochastic gradient Langevin algorithm (SGLD) for Monte Carlo Sampling~\citep{SGLD}. 
The purpose of this paper is to demonstrate that the existing convergence analysis of Langevin algorithm can be leveraged to derive generalization bound of the DP-SGD method with good differential privacy guarantees. Using some recently obtained dimension-independent convergence results for Lagenvin algorithms, we establish differential privacy analysis of DP-SGD that improves existing bounds in terms of dimension dependency for certain classes of smooth convex optimization problems.

\section{Related Work}

There have been significant interests in developing machine learning algorithms with differential privacy guarantees. In practice, the most popular algorithm is DP-SGD, and its differential privacy property has been investigated in \citep{abadi2016deep}. However, it is known that the requirement of differential privacy guarantee may negatively affect the generalization performance of a machine learning algorithm. Therefore in practice, 
the trade-off between differential privacy and generalization of SGD becomes an important consideration.
This issue was not considered in \citep{abadi2016deep}.

To better understand the relationship between differential privacy and generalization, more recent works focused on the theoretical analysis of differentially private algorithms for solving stochastic convex optimization problems. In particular, it was shown that for generic smooth stochastic convex optimization without additional structures, one can achieve a $O\lrp{\frac{1}{\sqrt{n}} + \frac{\sqrt{d}\log(1/\delta)}{\epsilon n}}$ excess generalization error bound for $(\epsilon,\delta)$-differential privacy (see Definition~\ref{def:dp} below for the definition of DP) \citep{bassily2019private}. The work of \cite{bassily2014private} show that the above rate is not improvable in general, when considering constrained optimization problems.
Additional results showed that similar results can be extended to nonsmooth problems and for less computational costs \citep{feldman2020private,Bassily2020}. Unfortunately, such a result means that in the overparameterized regime where $d \geq n$, any privacy guarantees at a level smaller than constant will lead to worse generalization results than that of SGD. Such results imply that DP-SGD significantly affects the performance of machine learning algorithms in the overparameterized regime. 

The purpose of this paper is to show with additional structures that are naturally satisfied for classes of machine learning problems such as logistic regression, the pessimistic dimension dependent excess generalization lower bound in \citep{bassily2019private} for smooth convex optimization problems can be avoided. In particular, we are able to obtain a dimension independent excess generalization bound at the optimal rate of $\tilde{O}(1/\sqrt{n})$ with a differential privacy guarantee at the level of $\epsilon=\tilde{O}(n^{-1/4})$. This implies that even for overparameterized models commonly used in practice, DP-SGD can produce differentially private machine learning models without negatively affecting the generalization performance.

\section{Main Results}

\begin{algorithm}[h]
\SetAlgoLined
\KwInput{decreasing sequence of stepsize $\eta_t$ and $\beta_0$}
Draw $w_1$ from $N(0,\beta_0 I)$ \\
\For{$t = 2, \ldots,T$}{
Sample mini-batch $\Mcal_{t-1} \subseteq [n]$ (without replacement).\\
Construct mini-batched gradient $\Gcal_{t-1}(w_{t-1}) = \frac{1}{|\mathcal{M}_{t-1}|}\sum_{i \in \mathcal{M}_{t-1}} \nabla \ell(w_{t-1},z_{i})$.
Let
\begin{align}
    \tw_{t-1} = {w}_{t-1} - {\eta}_t \Gcal_{t-1}(w_{t-1}) . \label{eq:sgd}
\end{align}
\begin{align}
\textrm{Sample} \quad    w_t \sim \mathcal{N}\lrp{ (1-\lambda_t\eta_t) \tw_{t-1},  \lambda_t\eta_t(2-\lambda_t\eta_t) \beta_0 I}.
    \label{eq:Langevin_sample}
\end{align}
 }
    \caption{DP-SGD (Stochastic Gradient Langevin Algorithm)}
    \label{alg:sgld}
\end{algorithm}

Mathematically, the DP-SGD algorithm corresponds to Algorithm~\ref{alg:sgld}, which is also referred to as Stochastic Gradient Langevin Dynamics (SGLD) in the MCMC sampling literature.

\begin{definition}
Let $\dataset_n \in \cZ^n$ be a dataset. A randomized learning algorithm $\cA$ maps $\dataset_n$ to 
an output $w\in \Omega$.
We say $\cA$ is $(\epsilon,\delta)$-DP (differentailly private) if for any two datasets $\dataset_n$ and $\dataset_n'$ differing in at most one item, the following statement holds for all $E \subset \Omega$:
\[
\textstyle
P(\cA(\dataset_n) \in E) \leq e^{\epsilon} P(\cA(\dataset_n') \in E) + \delta .
\]
\label{def:dp}
\end{definition}

 We are interested in the trade-off of generalization, differential privacy, and computation in DP-SGD. We consider two settings, one is the single-pass setting, and the other is multi-pass setting. 
Before presenting the main results, we introduce the following notations and assumptions.
 Define
\[
 \textstyle
 \ell(w,\cD)= \rE_{z \sim \cD} \ell(w,z) .
 \]
 Moreover for any distribution $p$ on $\rR^d$, define
 \[
 \ell(p,\cD) = \rE_{w \sim p} \ell(w,\cD) .
  \]
  
 We make the following assumptions.
 \begin{assumption}
 Each loss function $\ell(w,z)$ is convex in $w$. 
 \end{assumption}
\begin{assumption}
We assume that
\[
\textstyle
\sup_{w,z} \|\nabla_w \ell(w,z) \|_2 \leq G, \qquad
{\text and} \qquad
 \sup_w \nabla_w^2 \ell(w,\cD) \preceq \tilde{H}_\cD .
\]
\end{assumption}
\begin{assumption}
\label{assm:alg_minibatch}
We either run Algorithm~\ref{alg:sgld} as a multi-pass algorithm with $T\geq n$ and $|\Mcal_{t}| = 1,\forall t\in [T]$ or we run it as a single-pass algorithm with $|\Mcal_t| = \left\lceil\sqrt{\frac{T}{2(T-t+1)}}\right\rceil$.
\end{assumption}
Assumption~\ref{assm:alg_minibatch} implies that we can carry out the convergence analysis in the simpler \emph{online} setting as in the single-pass setting each $z_i$ is going to be observed only once and there will be a martingale structure which allows us to decouple $w_t$ from $\Gcal_t(w_t)$. In fact the size of the mini-batches only matters in the privacy analysis and for showing the generalization bound we could take $|\Mcal_t| = 1$.

\subsection*{Single Pass Setting}
As we already mentioned, in the single pass setting we to through the data only once, that is we use $z_i \in \dataset$ in a single gradient computation. We assume that the data size $n$ is known a priori. We have the following result. 
\begin{theorem}
\label{thm:single_pass_dp}
Assume that we run Algorithm~\ref{alg:sgld} for $T=n$ iterations in the single pass setting with $\eta_t = \sqrt{\frac{1}{4G^2T}}$, $|\Mcal_t| = \sqrt{\frac{T}{2(T-t+1)}}$, $\lambda_t = \frac{1}{t\eta_t}$, $\beta_0 = \frac{(\epsilon + \log(1/\delta)) \eta_0^2}{2\epsilon^2 T}$, the last iterate, $w_T\sim p_T$, of Algorithm~\ref{alg:sgld} is $(2\epsilon,\delta)$-DP and satisfies for all $w\in \mathbb{R}^d$:
\begin{align*}
\textstyle
    \rE[\ell(p_T,\cD)] \leq \ell(w,\cD) + \frac{\ln(n)\big(G \|w\|_2^2 + 1.5\eta_0^2G\big)}{\eta_0 \sqrt{n}}+ \frac{\eta_0^2\log(1/\delta) \trace(\tilde H_{\cD})}{\epsilon^2 n}.
\end{align*}
Further, Algorithm~\ref{alg:sgld} does not use more than $n$ samples to output $w_T\sim p_T$.
\end{theorem}

\subsection*{Multi Pass Setting}
Although single pass setting is sufficient for theoretical purposes, in practice, one often runs DP-SGD multiple times over the data. Another advantage of the mulit-pass analysis is that the DP guarantee is with respect to the all intermediate iterates, as opposed to the single pass setting, and thus suitable for distributed computing when intermediate iterates need to be shared across multiple machines. 
 Therefore we also include a multi-pass analysis for consistency with practice. 
 
\begin{theorem}
Let $\eta_t = G^{-1}\sqrt{\beta_0} /\sqrt{8 \ln(2.5t^2/n\delta)t}$ and $|\cM_{t-1}|=1$
for $t \geq 1$, $\lambda_1=1/\eta_1$ and $\lambda_t = \frac{1}{\eta_t} - \frac{1}{\eta_{t-1}}$ for $t \geq 2$.
Assume $T=n^\alpha \epsilon^2 \geq n$ for $1 \leq \alpha \leq 2$ and $\beta_0 = \eta_0^2 (n/T)$ for some $\eta_0>0$. Then Algorithm~\ref{alg:sgld} is  $(3\epsilon\sqrt{\ln(2/\delta)} + 3 \epsilon^2,\delta)$-DP 
with respect to all intermediate iterates $\{w_t\}$.  
Moreover, the following bound on the average generalization error holds for all $w \in \rR^d$:
\begin{align*}
\textstyle
\rE \; \frac1T \sum_{t=1}^T \ell(p_t,\cD)
&= \ell(w,\cD) +
{O} \lrp{ \frac{G \lrn{w}_2^2 \sqrt{\ln(T/\delta)}}{\eta_0 \sqrt{n}} 
+  \frac{\eta_0 G}{\sqrt{n \ln (1/\delta)}}
+ \frac{\eta_0^2\trace(\tilde{H}_\cD)}{\epsilon^2 n^{\alpha-1}} } ,
\end{align*}
where $O(\cdot)$ hides an absolute constant. 
\label{thm:multi_pass}
\end{theorem}

Making different choices of $T=n^\alpha \epsilon^2$ will lead to different generalization error with the same $(3\epsilon\sqrt{\ln(2/\delta)} + 3 \epsilon^2,\delta)$ differential privacy. 
Assume that $\trace(\tilde{H}_\cD)$ is bounded, then we take $\epsilon=n^{-1/4}$ and 
$\alpha=2$. This implies $(\epsilon,\delta)$-DP with $T=n^{1.5}$ computation, and excess generalization error of
$\tilde{O}(1/\sqrt{n})$.

We also note that the results of Theorem~\ref{thm:single_pass_dp}  and Theorem~\ref{thm:multi_pass} are comparable up to logarithmic factors. Although we are not concerned with logarithmic factors in this paper, we note that the logarithmic factor in Theorem~\ref{thm:single_pass_dp} can be improved 
by using the step-size schedule in \citep{jain2019making}, and the logarithmic factor  in Theorem~\ref{thm:multi_pass} can be improved by using the strong composition analysis with RDP instead of with DP as in this paper \citep{abadi2016deep,mironov2017renyi}.

\subsection{Example: Ridge separable problems}
\label{sec:ridge_sep}
We note that in the general setting where $\trace(\tilde{H}_\cD)$ depends on the dimensionality $d$, we can obtain the following generalization performance for DP-SGD
\[
\ell(w,\cD) + O\left(\frac{G\|w\|_2}{\sqrt{n}} + \frac{G\|w\|_2\sqrt{d\log(1/\delta)}}{\epsilon n}\right) ,
\]
which was shown by \citep{feldman2020private}. However, in many applications, the trace of the Hessian $\tilde H_{\Dcal}$ could be constant or much smaller compared to $d$. We now give one such example.

We consider the binary logistic regression (or smoothed SVM losses) as a concrete example, where 
\[
\textstyle
\ell(w,z) = \phi(w^\top x, y) ,
\]
and $z=(x,y)$. Assume $\|x\|_2 \leq 1$ for all $x$. 
$|\phi_1'(a,y)| \leq \gamma_1$, and 
$\phi_1''(a,y) \leq \gamma_2$, then 
\[
\textstyle
G \leq \gamma_1 , \quad
\trace(\tilde{H}_\cD) \leq  \gamma_2 .
\]
Theorem~\ref{thm:single_pass_dp} and Theorem~\ref{thm:multi_pass} now imply an excess generalization error of at most $\tilde O(n^{-1/2})$ for any setting of $\epsilon = \Omega (n^{-1/4})$.

In comparison, for overparameterized problems with $d \geq n$, we have excess generalization error of
\[
\textstyle
\tilde{O}\left(\frac{1}{\sqrt{n}} + \frac{\sqrt{d}}{\epsilon n}\right)
=\tilde{O}\left(\frac{1}{\epsilon \sqrt{n}}\right) 
\]
at the level of $(\epsilon,\delta)$ differential privacy from earlier results. Such results are infereior to our result for overaprameterized models used in practice. 
Our result is also closely related to those of \citep{song2020evading}, 
which also considered the problem of deriving dimensional independent results. Their results employ the rank of the Hessian instead of the trace of the Hessian as in this paper. In applications, rank may be as large as $O(n)$, while trace is generally small.
We note that the results in \citep{feldman2020private,song2020evading} with which we compare hold for the more general setting of constrained convex optimization, where $w\in \mathcal{P}$ for some convex set $\mathcal{P}$ describing the constraints. The results presented in this work hold only in the unconstrained setting. Obtaining bounds which depend on the trace of the Hessian in the constrained setting is an interesting open problem.
 
\section{Analysis of DP-SGD}
\label{sec:overview_proofs}
 Before presenting our analysis,  we introduce the following additional notations.
 Define
\[
 \textstyle
  \ell(w,\dataset_n) = \frac1n \sum_{i=1}^n \ell(w, z_i) .
\]
 Moreover for any distribution $p$ on $\rR^d$, define
 \[
  \ell(p,\dataset_n)= \rE_{w \sim p} \ell(w,\dataset_n) .
 \]
 Define for any distribution $q$ on $\rR^d$:
 $H(p\|q)=\rE_{w \sim p} \ln \frac{p(w)}{q(w)}$.
 
 We divide the main steps in our analysis into the convergence of the DP-SGD procedure, generalization analysis, and DP-analysis. We then put these results together to prove the main results of the paper. 
 The following summarizes the key intermediate steps. Detailed proofs of the required technical results can be found in Section~\ref{sec:proof}.
 
 \subsection{Proof of Theorem~\ref{thm:multi_pass}} 
 \label{sec:multi_pass_proof}
 
\subsubsection*{Convergence Analysis} 

The main convergence result for Algorithm~\ref{alg:sgld}, stated in the following proposition, is adapted from \citep{freund2021convergence}.

\begin{proposition} 
Let $p_t(w)$ denote the distribution of $w_t$, and let
$q(w) = N(0, \beta_0 I)$,
then
\begin{align*}
\textstyle
\sum_{t=1}^T \left[\ell(p_t,z_t) + 
 \beta_t  H(p_t\|q) \right] 
\leq \inf_p 
\sum_{t=1}^T \left[\ell(p,z_t)+ 
 \beta_t H(p\|q)\right]
+\sum_{t=1}^T \frac{\eta_t}{2} G^2 ,
\end{align*}
where we set $\lambda_1=1/\eta_1$; $\forall t \geq 1,
\lambda_t \geq \frac{1}{\eta_t} - \frac{1}{\eta_{t-1}}$;
and $\beta_t= \lambda_t \beta_0$. 
\label{prop:conv}
\end{proposition}

Bounding the term $\inf_p  \sum_{t=1}^T \left[\ell(p,z_t)+ 
 \beta_t H(p||q)\right]$ results in the following corollary. 
\begin{corollary}
\label{corollary:conv}
For Algorithm~\ref{alg:sgld}, we have the following regret bound for any fixed $w \in \rR^d$:
\[
\textstyle
\rE \; \sum_{t=1}^T \left[ \ell(p_t,\dataset_n)
-\ell(w,\dataset_n)\right]
 \leq \frac{\tilde{\lambda}_T \|w\|_2^2}{2} + \tilde{\lambda}_T\beta_0  \ln \left|\frac{T}{\tilde{\lambda}_T} \tilde{H}_\cD + I \right| 
  + \sum_{t=1}^T \frac{\eta_t}{2} G^2 ,
\]
where the expectation is with respect to the randomization in the algorithm , and
$\tilde{\lambda}_T=\sum_{t=1}^T \lambda_t$. 
\end{corollary}

\subsubsection*{Generalization Analysis}

In order to obtain optimal trade-off of generalization and differential privacy in the setting when $|\Mcal_t|=1,\forall t\in [T]$, we need to run DP-SGD more than once over the training dataset. In this situation, the standard online learning analysis doesn't apply. It is also difficult to employ standard empirical process theory to analyze such systems in the general setting due to the lack of general covering number results.
 In this section, we derive generalization bounds using the stability analysis for SGD of \citep{hardt2016train}, which was also employed by
  \citep{Bassily2020} to analyze DP-SGD. 
  For simplicity, we only consider the expected version, with proofs included for completeness.
  
  \begin{proposition}
  Given any fixed $w$ on $\rR^d$, we have for all $t$:
  \[
  \textstyle
  \rE \ell(p_t,\cD) 
  \leq \ell(w,\cD)
  + \rE
  [\ell(p_t,\dataset_n) -\ell(w,\dataset_n) ]
  + \frac{2 G^2}{n} 
  \sqrt{(n+t)\sum_{s \leq t} \eta_s^2} . 
  \]
  \label{prop:generalization}
  \end{proposition}
  
 We can now combine Proposition~\ref{prop:generalization} and Corollary~\ref{corollary:conv} to obtain the following generalization result. 
 
 \begin{corollary}
 \label{cor:generalization}
 For Algorithm~\ref{alg:sgld}, we have the following generalization bound:
\begin{align*}
\rE \; \frac1T \sum_{t=1}^T \ell(p_t,\cD) 
&\leq \inf_{w} \left[\ell(w,\cD) + \frac{\|w\|_2^2}{2T \eta_T} \right] + \frac{\beta_0}{T \eta_T} \ln \left|\eta_T T \tilde{H}_\cD + I \right|\\
  &+ \sum_{t=1}^T \frac{\eta_t G^2}{2T} 
  + \frac{2 G^2}{n} 
  \sqrt{(n+T)\sum_{t \leq T} \eta_t^2} .
\end{align*}
 \end{corollary}

\subsubsection*{DP-analysis}

 Our analysis follows that of 
\citep{bassily2014private}.
Although similar to \citep{abadi2016deep,Bassily2020}, we can remove a $\ln (1/\delta)$ factor by using the R\'enyi entropy approach of 
\citep{mironov2017renyi}. We consider a simpler derivation below using the strong composition theorem from \citep{dwork2010boosting}, together with DP-guarantee of Gaussian mechanism, and amplification theorem of  
\citep{Balle_dp_amplification} to handle SGD.

\begin{proposition}
\label{proposition:dp_sgd}
If we take $\eta_t= G^{-1}\sqrt{\beta_0} /\sqrt{8 \ln(2.5t^2/n\delta)t}$, then
Algorithm~\ref{alg:sgld}, using sampling with replacement and setting $|\Mcal_t|=1$, is
 \[
 \textstyle
 \left(\sqrt{2 T\ln(2/\delta)}\; e/n
  + T e^2/n^2 ,\delta\right)
 \]
 differentially private.
 \end{proposition}
\subsubsection*{Proof of Theorem~\ref{thm:multi_pass}}
Taking $\eta_t$ as in Proposition~\ref{proposition:dp_sgd} implies the DP part of the result. 
Further, combining this choice with Corollary~\ref{cor:generalization} implies a generalization bound of (assuming $T \geq n$):
 \begin{align*}
 \rE \; \frac1T \sum_{t=1}^T \ell(p_t,\cD) 
\leq& \ell(w,\cD) + \frac{\|w\|_2^2}{2T \eta_T} + \frac{\beta_0}{2T \eta_T} \ln \left|\eta_T T \tilde{H}_\cD + I \right| 
  + \sum_{t=1}^T \frac{\eta_t G^2}{2T} 
  + \frac{2 G^2}{n} 
  \sqrt{(n+T)\sum_{t \leq T} \eta_t^2} \\
  =& \ell(w,\cD) + O\left(
  \frac{G\|w\|_2^2 \sqrt{\ln (T/\delta)}}{2 \sqrt{T\beta_0}} + \beta_0 \trace(\tilde{H}_\cD) 
  + \frac{G \sqrt{\beta_0 T}}{n\sqrt{\ln (1/\delta)}} 
   \right) ,
\end{align*}  
where we have used the inequality
$\ln |I + A| \leq \trace(A)$, and $O(\cdot)$  hides an absolute constant. 
This bound implies Theorem~\ref{thm:multi_pass}.

\subsection{Proof of Theorem~\ref{thm:single_pass_dp}}
\subsubsection*{Convergence and Generalization Analysis}
In the single-pass mini-batched setting we can obtain a simple generalization guarantee by taking expectation for the averaged iterate and using Corollary~\ref{corollary:conv}. The average iterate, however, may not enjoy good privacy properties. Following~\citep{shamir2013stochastic} it is possible to show a last iterate convergence guarantee stated below.
\begin{corollary}
\label{cor:single_pass_gen}
For the single-pass setting in Algorithm~\ref{alg:sgld} with fixed step-size $\eta_t = \eta, \lambda_t = \frac{1}{t\eta}, \forall t\in[T]$ it holds that 
\begin{align*}
\textstyle
    \rE\left[\ell(p_T,\Dcal)\right] - \ell(w,\Dcal) \leq \frac{\ln(T)\|w\|_2^2}{T\eta} + \beta_0\trace(\tilde H_{\Dcal}) + \frac{3\ln(T)}{2}\eta G^2,
\end{align*}
for any fixed $w\in \mathbb{R}^d$.
Further, if the mini-batch size is set as $|\Mcal_t| = \left\lceil\sqrt{\frac{T}{2(T-t+1)}}\right\rceil$, then the number of samples used by Algorithm~\ref{alg:sgld} is at most $T$.
\end{corollary}

\subsubsection*{DP Analysis}
In the single-pass setting our analysis follows that of~\citep{feldman2020private}. To state our result formally we first need to introduce the notion of R\'enyi differential privacy.
\begin{definition}
Let $\mu$ and $\nu$ be two measures on $\Xcal$ which are absolutely continuous with respect to the Lebesgue measure with densities $f_\mu$ and $f_\nu$. The $\alpha$-R\'enyi divergence, for $1<\alpha \leq \infty$ between $\mu$ and $\nu$ is defined as $D_\alpha(\mu\|\nu) = \frac{1}{\alpha-1}\log\int_{\Xcal} \left(\frac{f_\mu(x)}{f_\nu(x)}\right)^\alpha f_{\mu}dx,$
if $\mu \ll \nu$ and $D_\alpha(\mu\|\nu) = \infty$ otherwise.
\end{definition}

\begin{definition}[\citep{mironov2017renyi}]
Let $\Acal : \Dcal \to \RR$ be a randomized algorithm and let $\Pcal_\Acal : \Dcal \to \PP(\RR)$ be the induced probability measure by $\Acal$. $\Acal$ is $(\alpha,\epsilon)$-R\'enyi differentially private (RDP) if for any two neighboring datasets $D,D' \in \Dcal$ it holds that $D_{\alpha}(\Pcal_\Acal(D)\|\Pcal_\Acal(D')) \leq \epsilon$.
\end{definition}

\begin{proposition}
\label{prop:rdp}
Algorithm~\ref{alg:sgld} in the single-pass setting with $\eta_t = \eta$, $\lambda_t = \frac{1}{t\eta}$, and $|\Mcal_t| = \left\lceil\sqrt{\frac{T}{2(T-t+1)}}\right\rceil$ is $(\alpha,\epsilon)$-RDP with
$\epsilon = 2\frac{\alpha \eta^2 G^2}{\beta_0}$,
for the last iterate.
\end{proposition}

\subsubsection*{Proof of Theorem~\ref{thm:single_pass_dp}}
Fix any $w\in \mathbb{R}^d$. Using the results in Corollary~\ref{cor:single_pass_gen} and Proposition~\ref{prop:rdp} with $|\Mcal_t| = \sqrt{\frac{T}{2(T-t+1)}}$, $\beta_0 = \frac{\alpha\eta_0^2}{2\epsilon T}$, $\eta = \eta_0\sqrt{\frac{1}{4G^2T}}$, and $\lambda_t = \frac{1}{t\eta}$ implies Algorithm~\ref{alg:sgld} is $(\alpha,\epsilon)$-RDP for the last iterate and satisfies
\begin{align*}
    \rE[\ell(p_T,\cD)] - \ell(w,\cD) \leq \frac{\ln(T)(\frac{\|w\|_2^2G}{\eta_0} + \frac{3}{2}\eta_0G)}{\sqrt{T}} + \frac{\alpha\eta_0^2\trace(\tilde H_{\cD})}{2\epsilon T}.
\end{align*}

The following lemma allows us to convert the RDP result into a differential privacy guarantee.
\begin{lemma}[\citep{mironov2017renyi}]
If an algorithm $\Acal$ is $(\alpha,\epsilon)$-RDP then for any $\delta>0$ it satisfies $(\epsilon + \frac{\log(1/\delta)}{\alpha-1}, \delta)$ differential privacy (DP).
\end{lemma}

Setting $\alpha = 1+ \frac{\log(1/\delta)}{\epsilon}$, we obtain the desired result.

\section{Proofs of Technical Lemmas and Propositions}
\label{sec:proof}
\subsection{Convergence}
\begin{proof}[Proof of Proposition~\ref{prop:conv}]
Define: 
$H(p) = \rE_{w \sim p} \ln p(w)$.
We also let $g_t(w) = \frac{\lambda_t}{2} \|w\|_2^2$.
We additionally use the following definition.
 \begin{definition}
Given two probability distributions $p(u)$ and $p'(v)$ on $\rR^d$, and 
let $\Pi(p,p')$ be the class of distributions $q(u,v)$ on $\rR^d \times \rR^d$ so that the marginals
$q(u)=p(u)$ and $q(v)=p'(v)$. 
The $W_2$ Wasserstein distance of $p$ and $p'$ is defined as
\[
\textstyle
 W_2(p,p')^2 = \min_{q \in \Pi(p,p')}
 \rE_{(u,v) \sim q} \|u-v\|_2^2  .
\]
\end{definition}

Then we have the following two lemmas from \citep{freund2021convergence}.

\begin{lemma}[For Regularized Entropy]
We  have for $t \geq 1$:
\begin{align*}
\textstyle
&2 \eta_t \left[ [\rE_{w \sim p_t} g_t(w) + \beta_t H(p_t)] - [\rE_{w \sim p} g_t(w) + \beta_t H(p)]
\right] \\
\leq&  [(1-\lambda_t \eta_t) W_{2}({\tp}_{t-1},p)^2 - W_2(p_{t},p)^2 ] ,
\end{align*}
where $\tp_{t-1}$ is the distribution of $\tw_{t-1}$.
Note that $1-\lambda_1\eta_1=0$, so we do not have to define $\widetilde{p}_0$. 
\label{lem:entropy}
\end{lemma}
\begin{lemma}
We have for all $t \geq 1$:
\begin{align*}
\textstyle
&2 \eta_t [
\rE_{w \sim p_t} \ell(w,z_t)
- \rE_{w \sim p} \ell(w,z_t) ]\\
\leq &
[W_2(p_{t},p)^2
- W_{2}({\tp}_{t},p)^2]
+ \eta_t^2 \rE_{w \sim {p}_t}
\|\nabla \ell_t(w)\|_2^2\\
\leq &
[W_2(p_{t},p)^2
- W_{2}({\tp}_{t},p)^2]
+ \eta_t^2 G^2 .
\end{align*}
\label{lem:loss}
\end{lemma}
By summing the inequalities in Lemma~\ref{lem:entropy} and Lemma~\ref{lem:loss}, and divide the resulting inequality by
$\eta_t$, we obtain
\begin{align*}
\textstyle
&2\rE_{w \sim p_t} [\ell(w,z_t) + g_t(w) + \beta_t \ln p_t(w)] - 2\rE_{w \sim p} [\ell(w,z_t)+g_t(w) + \beta_t \ln p(w)] \\
\leq& \frac{1}{\eta_{t-1}} W_2(\tp_{t-1},p)^2 - \frac{1}{\eta_{t}} W_2(\tp_t,p)^2 + {\eta_t G^2} .
\end{align*}
We can now sum over $t=1$ to $t=T$:
\begin{align*}
\textstyle
\sum_{t=1}^T \rE_{w \sim p_t} \left[\ell(w,z_t) + 
 \beta_t \ln \frac{p_t(w)}{q(w)}\right] 
\leq \inf_p \rE_{w \sim p}
\sum_{t=2}^T \left[\ell(w,z_t)+ 
 \beta_t \ln \frac{p(w)}{q(w)}\right]
+\sum_{t=1}^T \frac{\eta_t}{2} G^2 .
\end{align*}
\end{proof}

\begin{proof}[Proof of Corollary~\ref{corollary:conv}]
We employ the following estimate. 
\begin{lemma}
\label{lem:Q_bound}
Let $\sum_{t=1}^T \beta_t = \tilde{\lambda}_T \beta_0$, and let $Q_T(w) = \sum_{t=1}^T 
\ell(w,S_n) + \frac{\tilde{\lambda}_T}{2}\|w\|_2^2$.
We have
\begin{align*}
\textstyle
\rE \inf_p \rE_{w \sim p} \sum_{t=1}^T\left[ \ell(w,z_t)+ \beta_t \ln \frac{p(w)}{q(w)}\right]
\leq& \inf_w \rE Q_T(w) + \tilde{\lambda}_T \beta_0\ln \left|\frac{T}{\tilde{\lambda}_T} \tilde{H}_\cD + I \right| ,
\end{align*}
where the expectation is with respect to the randomization in the SGLD steps. 
\end{lemma}
\begin{proof}
Let $w_*=\arg\min_w Q_T(w)$.
Combining the smoothness assumption and the fact that we can upper bound the Hessian of $Q_T$ as $\tilde H + \tilde{\lambda}_T I$ we have
\begin{equation}
\textstyle
Q_T(w) \leq Q_T(w_*) + \frac{1}{2} (w-w_*)^\top \tilde{H} (w-w_*) + \frac{\tilde{\lambda}_T}{2} \|w-w_*\|_2^2 .
\label{eq:Q-smoothness}
\end{equation}
Therefore
\begin{align*}
\textstyle
&\inf_p \rE\; \rE_{w \sim p} \sum_{t=1}^T\left[ \ell(w,z_t)+ \beta_t \ln \frac{p(w)}{q(w)}\right]\\
=& - \tilde{\lambda}_T \beta_0
\ln \rE_{w \sim q}
\exp \left(- \frac{1}{\tilde{\lambda}_T \beta_0} \rE\; \sum_{t=1}^T \ell(w,z_t)\right) \\
=& - \tilde{\lambda}_T \beta_0
\ln \int \; \frac{1}{(2\pi \beta_0)^{d/2}}
\exp \left(- \frac{1}{\tilde{\lambda}_T \beta_0} Q_T(w)\right) d w \\
\leq& Q_T(w_*) - \tilde{\lambda}_T \beta_0
\ln \int \frac{1}{(2\pi \beta_0)^{d/2}}
\exp \left(-\frac{1}{2\tilde{\lambda}_T\beta_0} \sum_{t=1}^T (w-w_*)^\top \tilde{H}_\cD (w-w_*) - \frac{1}{2\beta_0} \|w-w_*\|_2^2 \right) d w \\
\leq& Q_T(w_*) + \tilde{\lambda}_T \beta_0 \ln \bigg| \frac{T}{\tilde{\lambda}_T} \tilde{H}_\cD + I \bigg|
  .
\end{align*}
The inequality is a consequence of \eqref{eq:Q-smoothness}. 
\end{proof}

We can plug Lemma~\ref{lem:Q_bound} into the right hand side of Proposition~\ref{prop:conv}, and note that
\[
\textstyle
\beta_t \rE_{w \sim p_t} \ln \frac{p_t(w)}{q(w)}
\geq 0 .
\]
This implies Corollary~\ref{corollary:conv}.
\end{proof}

\subsubsection{Last iterate convergence}
\begin{proof}[Proof of Corollary~\ref{cor:single_pass_gen}]
We proceed as in \citep{shamir2013stochastic}. Let $f(p) = \rE_{w\sim p,z_t\sim \Dcal}[\ell(w,z_t) + g_t(w)] + \beta_0 H(p)$. Lemma~\ref{lem:entropy} and Lemma~\ref{lem:loss} imply
\begin{align*}
\textstyle
    \sum_{t=T-k}^T f(p_t) - (k+1)f(p) \leq \rE\left[\frac{W_2^2(\tilde p_{T-k-1},p)}{2\eta}\right] + (k+1) G^2\frac{\eta}{2},
\end{align*}
for any $p$ measurable with respect to the sigma-algebra induced by the randomness of the updates in the first $T-k$ rounds. We set $p=p_{T-k-1}$ which implies
\begin{align*}
\textstyle
    -f(p_{T-k-1}) \leq \frac{1}{(k+1)}\left(\rE\left[\frac{W_2^2(\tilde p_{T-k-1},p_{T-k-1})}{2\eta}\right] + (k+1) G^2\frac{\eta}{2} - \sum_{t=T-k}^T f(p_t)\right).
\end{align*}
Let $\Sfrak_k = \frac{1}{k+1}\sum_{t=T-k}^T f(p_t)$. Then we have
\begin{align*}
\textstyle
    (k+1)\Sfrak_k &= \sum_{t=T-k}^T f(p_t) = \sum_{t=T-k-1}^T f(p_t) - f(p_{T-k-1}) = (k+2)\Sfrak_{k+1} - f(p_{T-k-1})\\
    &\leq (k+2)\Sfrak_{k+1} - \Sfrak_{k} + \rE\left[\frac{W_2^2(\tilde p_{T-k-1},p_{T-k-1})}{2\eta(k+1)}\right] + G^2\frac{\eta}{2}\\
    \iff \qquad \qquad \;\;\, \Sfrak_k &\leq \Sfrak_{k+1} + \rE\left[\frac{W_2^2(\tilde p_{T-k-1},p_{T-k-1})}{2\eta(k+1)(k+2)}\right] + G^2\frac{\eta}{2(k+2)}\\
    \implies \quad f(p_T) = \Sfrak_0 &\leq \Sfrak_T + \sum_{k=0}^{T-1} \rE\left[\frac{W_2^2(\tilde p_{T-k-1},p_{T-k-1})}{2\eta(k+1)^2}\right] + \frac{\eta G^2\ln(T)}{2}.
\end{align*}
Finally we bound $W_2^2(\tilde p_t,p_t) \leq \rE_{\tilde w_{t}\sim \tilde p_t}[\|\tilde w_t - \eta\nabla \ell(\tilde w_t,z_t) -  \tilde w_t\|^2] \leq \eta^2 G^2$.
Combining with our bound on $\Sfrak_T$ from Corollary~\ref{corollary:conv} we have
that for any $w$ it holds that
\begin{align*}
    &\rE[\rE_{w_{p_T}}\left[\ell(w_{p_T},\Dcal) + g_T(w_{p_T})\right]] - \ell(w,\Dcal) - g_T(w)  + \beta_0 H(p_T\|q)\\
    \leq &\frac{\tilde\lambda_T\|w\|_2^2}{2T} + \frac{\tilde\lambda_T\beta_0}{T} \ln \left|\frac{T}{\tilde \lambda_T} \tilde{H}_\cD + I \right| + \frac{3}{2}\eta G^2\ln(T).
\end{align*}
Using the setting $\lambda_t = \frac{1}{t\eta}$ we can bound $\tilde \lambda_T \leq \frac{\ln(T)}{\eta}$. Together with the fact $\beta_0H(p_T||q) \geq 0$ we have
\begin{align*}
    \rE\left[\ell(p_T,\Dcal)\right] - \ell(w,\Dcal) \leq \frac{\ln(T)\|w\|_2^2}{T\eta} + \beta_0\trace(\tilde H_{\Dcal}) + \frac{3\ln(T)}{2}\eta G^2.
\end{align*}
\end{proof}

\subsection{Generalization}
\begin{proof}[Proof of Proposition~\ref{prop:generalization}]
We need the following lemma for the stability of SGLD. 
\begin{lemma}
  Consider $S_n=\{z_1,\ldots,z_n\}$, and $S_{n}'=\{z_1,\ldots,z_{n-1},z_{n}'\}$. 
  Let $p_t$ and $p_t'$ be the distributions of model parameter after $t$ iterations of Algorithm~\ref{alg:sgld} with $|\cM_{s-1}|=1$ for all $s$. 
  Assume that $\ell(w,z)$ is $L$-smooth for all $z$ for some $L >0$, and $\eta_1 \leq 1/L$.
  Then for all $t$:
  \[
  \textstyle
  W_2(p_t,p_t')^2 \leq 4G^2(t/n^2+1/n)
  \sum_{s \leq t} \eta_s^2 .
  \]
  \label{lem:stability}
  \end{lemma}
  \begin{proof}[Proof of Lemma~\ref{lem:stability}]
  Consider $w_s$ ($\widetilde{w}_s$), and $w_s'$ ($\widetilde{w}_s'$) obtained from the algorithm with $S_n$ and $S_n'$, under the same randomization (except when we draw $z_n$ from $S_n$, we will draw $z_n'$ from $S_n'$). This gives a specialized couple of $p_s$ and $p_s'$. We can track $\rE \|w_s-w_s'\|_2^2$, which gives an upper bound of $W_2(p_s,p_s')^2$. 
  
  Consider any stage $t=s$.
  If we draw $\cM_{s-1}$ with $|cM_{s-1}|=1$ that does not contain $n$. Then we have the updates
  \begin{align*}
  \textstyle
  \widetilde{w}_{s-1}=&w_{s-1}- \eta_t \nabla\ell(w_{s-1},z_i)\\
  \widetilde{w}_{s-1}'=&w_{s-1}'- \eta_t \nabla\ell(w_{s-1}',z_i) .
  \end{align*}
  Since $\eta_t \leq 1/L$, we have
  \[
  \textstyle
  \|\widetilde{w}_{s-1}-\widetilde{w}_{s-1}'\|_2
  =\|[w_{s-1}-w_{s-1}']-\eta_t [\nabla\ell(w_{s-1},z_i)-\nabla\ell(w_{s-1}',z_i)]\|_2
  \leq \|w_{s-1}-w_{s-1}'\|_2 .
  \]
  If we draw $\cM_{s-1}$ (with $|\cM_{s-1}|=1$) that contains $n$ (and thus $z_n$ and $z_n'$ for the two datasets) instead, then we have
  the update
  \begin{align*}
  \|\widetilde{w}_{s-1}-\widetilde{w}_{s-1}'\|_2
  =&\left\|[w_{s-1}-w_{s-1}']- \eta_s [\nabla\ell(w_{s-1},z_n)-\nabla\ell(w_{s-1}',z_n')]\right\|_2\\
  =& \left\|w_{s-1}-w_{s-1}'- \eta_s \delta_s\right\|_2 ,
  \end{align*}
  where $\|\delta_s\|_2 =\|\nabla\ell(w_{s-1},z_n)-\nabla\ell(w_{s-1}',z_n')\|_2\leq 2 G$. 
  Since the probability that  $\cM_{s-1}$ that contains $n$ is $1/n$,
    we thus obtain
  \begin{align*}
  \textstyle
  &\rE \|\widetilde{w}_{s-1}-\widetilde{w}'_{s-1}\|_2^2\\
  \leq& \rE\|w_{s-1}-{w}_{s-1}'\|_2^2
  + \frac{4 \eta_s G}{n} \rE \|w_{s-1}-{w}_{s-1}'\|
  + \frac{4 \eta_s^2 G^2}{n} \\
  \leq& \frac{\eta_{s-1}^2}{\eta_s^2}
  \rE\|w_{s-1}-{w}_{s-1}'\|_2^2
  + \frac{\eta_{s}^2}{\eta_{s-1}^2-\eta_s^2}
  \frac{4 \eta_s^2 G^2}{n^2} 
  +\frac{4 \eta_s^2 G^2}{n} .
  \end{align*}
  The last inequality used $2ab \leq \frac{\eta_{s-1}^2-\eta_s^2}{\eta_{s}^2} a^2 + \frac{\eta_{s}^2}{\eta_{s-1}^2-\eta_s^2} b^2$.
  Since with the same randomization, 
  \[
  \rE\|w_{s}-{w}_{s}'\|_2^2
  = (1-\lambda_s\eta_s)^2 \rE \|\widetilde{w}_{s-1}-\widetilde{w}'_{s-1}\|_2^2=\frac{\eta_s^2}{\eta_{s-1}^2}
  \rE \|\widetilde{w}_{s-1}-\widetilde{w}'_{s-1}\|_2^2,
  \]
  and $\frac{\eta_{s}^2}{\eta_{s-1}^2-\eta_s^2}\leq s$, 
  we obtain
  \[
  \textstyle
  \rE\|w_{s}-{w}_{s}'\|_2^2
  \leq \rE\|w_{s-1}-{w}_{s-1}'\|_2^2
  + \frac{4t \eta_s^2 G^2}{n^2}
  +\frac{4 \eta_s^2 G^2}{n} .
  \]
  {We may take $w_1=w_1'$ because they both come from distribution $N(0,\beta_0 I)$. 
  By summing over $s=2$ to $s=t$, we obtain the desired result.}
  \end{proof}
We are now ready to prove Proposition~\ref{prop:generalization}. Consider two independent samples $\dataset_n=\{z_1, \ldots, z_n\}$, and $\dataset_n'=\{z_1', \ldots, z_n'\}$. Let
 $\dataset_n^{(i)}=\{z_1,\ldots, z_{i-1}, z_i', z_{i+1}, \ldots, z_n\}$. 
 Let $p_t^{(i)}$ be the distribution obtained from Algorithm~\ref{alg:sgld} using $\dataset_n^{(i)}$, $p_t$ be the distribution using $S_n$.
  We have
 \begin{align*}
 \textstyle
     \rE [\ell(p_t,\cD) -\ell(w,\cD)]
     = &\rE \frac{1}{n} \sum_{i=1}^n
     [\ell(p_t^{(i)}, z_i)
     - \ell(p_t, z_i)] 
     + \rE [\ell(p_t,\dataset_n)- \ell(w,\dataset_n) ] \\
      = &\rE \frac{1}{n} \sum_{i=1}^n
     \rE_{(w',w'') \sim q_*} [\ell(w', z_i)
     - \ell(w'', z_i)] 
     + \rE [\ell(p_t,\dataset_n)- \ell(w,\dataset_n) ]\\
     \leq& G
     \rE \frac{1}{n} \sum_{i=1}^n \rE_{(w',w'') \sim q_*}
     \|w'-w''\|_2
     + \rE [\ell(p_t,\dataset_n)- \ell(w,\dataset_n) ] \\
     \leq&
     \rE \frac{1}{n} \sum_{i=1}^n G
      W_2(p_t^{(i)},p_t)
     + \rE [\ell(p_t,\dataset_n)- \ell(w,\dataset_n) ] .
 \end{align*}
 In the above derivation, the first equality used the fact that $\ell(p_t,z)$ with $z \sim D$ has the same distribution as that of $\ell(p_t^{(i)},z_i)$.
 In the second equality, $q_*(w',w'')$ is the optimal $W_2$ coupling in $\Pi(p_t^{(i)},p_t)$ to achieve the optimality in $W_2$ distance. 
 The first inequality used the fact that $\ell(w,z)$ is  Lipschitz in $w$.
 The second inequality used the definition of $W_2$-distance and the Jensen's inequality. 
 Since $\dataset_n^{(i)}$ and $\dataset_n$ differ by one element, we obtain the desired result by applying Lemma~\ref{lem:stability}.
\end{proof}

\subsection{Differential Privacy}
We first state the following result.
\begin{lemma}
Consider $\dataset_m=\{z_{1},z_{2},\ldots,z_{m}\}$, and a randomized algorithm $w \sim \cA(\tilde{w},\dataset_m)$ with the one-step  update rule as follows
\[
\textstyle
w = \alpha \tilde{w} - \eta \nabla \ell(\tilde{w},z) + \sigma \epsilon, \qquad z \sim U(\dataset_m) , \quad \epsilon \sim N(0,I) ,
\]
where $U(\dataset_m)$ denotes the uniform sampling over elements of $\dataset_m$.
Then $\cA$ is
\[
\textstyle
\left(\ln\left(1+m^{-1}\exp\left(\sqrt{8\ln(1.25/\delta)} \eta G/\sigma\right)\right), \delta/m\right)
\]
deferentially private.
\label{lem:dp-random-onestep}
\end{lemma}
\begin{proof}
Consider $z, z' \in \cZ$, and the update rule
\begin{equation}
\textstyle
w = \alpha \tilde{w} - \eta \nabla \ell(\tilde{w},z) + \sigma \epsilon,  \quad \epsilon \sim N(0,I) 
\label{eq:proof-dp-random-onstep-1}
\end{equation}
with probability distribution $p(w)$, and 
\[
\textstyle
w' = \alpha \tilde{w} - \eta \nabla \ell(\tilde{w},z') + \sigma \epsilon,  \quad \epsilon \sim N(0,I) 
\]
with probability distribution $p'(w')$. 

We know from Theorem A.1 of \citep{dwork2014algorithmic}
that
the Gaussian mechanism \eqref{eq:Langevin_sample}
applied to any dataset (without considering the randomization effect) is 
$(\epsilon(\delta), \delta)$
differentially private, where $\epsilon(\delta)= \sqrt{8\ln(1.25/\delta)} \eta G/\sigma$. 
We can now apply the privacy amplification result, Theorem 9 of
\citep{Balle_dp_amplification}, to deal sampling without replacement, and
obtain the desired differential privacy guarantee for $\cA$.
\end{proof}

We can now use the strong composition theorem to obtain the following result. We state the following slightly more general result of  Theorem 3.3 of \citep{dwork2010boosting}, which allows step dependent $(\epsilon,\delta)$. 
\begin{proposition}
Consider an online learning algorithm $\cA: \cZ^n \to \Omega^T$.
If at each step $t$, $\cA([w_1,\ldots,w_{t-1}],\dataset_n)$ is 
$(\epsilon_t, 0.5\delta/t(t-1))$ differentially private for $t \geq 2$, then 
$\cA$ is
\[
\textstyle
\left( \sqrt{2\ln(2/\delta)\sum_{s=1}^T \epsilon_s^2} + \sum_{s=1}^T \epsilon_s(e^{\epsilon_s}-1),\delta\right)
\]
differentially private. 
\label{prop:composition}
\end{proposition}
\begin{proof}
  We simply note that:
 $\frac{\delta}{2} + \sum_{t \geq 2} \frac{0.5\delta}{t(t-1)} = \delta$.
The proof is essentially the same as that of \citep{dwork2010boosting}.
\end{proof}
We then have the following proof of Proposition~\ref{proposition:dp_sgd}. 
\begin{proof}[Proof of Proposition~\ref{proposition:dp_sgd}]
We note that Algorithm~\ref{alg:sgld} can be regarded as the composition of $T$ steps of mechanisms in Lemma~\ref{lem:dp-random-onestep}, where at each step $t$, 
we can take $\eta=\eta_t^2/\eta_{t-1}$, and $\sigma^2=(1-(\eta_t/\eta_{t-1})^2)\beta_0$ to obtain
$\left(\epsilon_t,  0.5\delta/t(t-1)\right)$
differential privacy 
for each step $t \geq 2$ of Algorithm~\ref{alg:sgld},
where
\[
\textstyle
\epsilon_t=\ln\left(1+ n^{-1}\exp\bigg(\sqrt{\frac{8\ln(2.5 t(t-1)/(m\delta))}{(\eta_t^{-2}-\eta_{t-1}^{-2})\beta_0}} \; G\bigg)\right) \leq \ln (1+\epsilon_0) ,
\]
where $\epsilon_0=e/n$. In the derivation, we used $t \ln t^2 - (t-1)\ln (t-1)^2 \geq \ln t^2$.
By applying Proposition~\ref{prop:composition}, we obtain 
\[
\textstyle
\left(\sqrt{2T \ln (2/\delta)} \ln (1+\epsilon_0) + T(\ln(1+\epsilon_0)(e^{\ln (1+\epsilon_0)}-1),\delta\right)
\]
differential privacy. Since
$\ln (1+\epsilon) \leq \epsilon$,  
we obtain the desired bound.
\end{proof}

\subsection{R\'enyi Differential Privacy}
We are going to use Theorem 3.1 from \citep{feldman2020private}. Before we state the theorem, however, we have to recall the update for the projected noisy stochastic gradient descent with mini-batches algorithm from \citep{feldman2020private} given below:
\begin{align*}
\textstyle
    w_t = w_{t-1} - \eta_t\left(\Gcal(w_{t-1}) + \xi_{t-1}\right),
\end{align*}
where $\xi_{t-1}\sim \Ncal(0,\sigma_{t-1}^2I)$.  
\begin{theorem}[Theorem 3.1\citep{feldman2020private}]
\label{thm:fedlman_dp}
  Let $\Kcal\subseteq \RR^d$ be a convex set and $\{f(\cdot,z)\}_{z\in \Zcal}$ be convex, $G$-Lipstchitz and $L$-smooth functions over $\Kcal$. Then for every mini-batch sequence $\{\Mcal_t\}_{t\in[T]}$, step-size sequence $\{\eta_t\}_{t\in[T]}$ with $\eta_t \leq \frac{2}{L},\forall t\in[T]$, noise parameter $\{\sigma_t\}_{t\in T}$, $\alpha\geq 1$, starting point $w_0$ and $S_n \in \Zcal^n$ it holds that projected noisy stochastic gradient descent with mini-batches (PNSGD) with the above parameters satisfies
  \begin{align*}
  \textstyle
      D_{\alpha}(w_T\|w'_T) \leq \max_{\tau\in[T]}\frac{2\alpha G^2 \eta_{\tau}^2}{|\Mcal_{\tau}|^2\sum_{t=\tau}^T \eta_t^2\sigma_t^2}.
  \end{align*}
\end{theorem}
\begin{proof}[Proof of Proposition~\ref{prop:rdp}]
While Algorithm~\ref{alg:sgld} is not exactly an instance of PNSGD we argue that we can apply Theorem~\ref{thm:fedlman_dp} with step size $\eta_t = \eta$ and $\sigma_t^2 = \frac{\eta\lambda_t\beta_0(2-\eta_t\lambda_t)}{\eta^2}$. To apply Theorem~\ref{thm:fedlman_dp} it is sufficient to argue that the update in Algorithm~\ref{alg:sgld} is contractive for any step-size sequence with $\eta_t \leq \frac{2}{L}$. We can write the update as
\begin{align*}
    w_{t+1} = (1-\lambda_t\eta_t)\left(w_t - \eta_t\Gcal(w_t)\right) + \xi_t,
\end{align*}
where $\xi_t \sim \Ncal(0, \lambda_t\eta_t(2-\lambda_t\eta_t)\beta_0I)$.
For two neighboring sequences $(w_t)_{t\in[T]}$ and $(w_t')_{t\in[T]}$ we have
\begin{align*}
    \|w_{t+1} - w_{t+1}'\|_2 \leq (1-\lambda_t\eta_t)\|w_t - \eta_t\Gcal(w_t) - w_t' - \eta_t\Gcal(w_t') \|_2 < \|w_t - w_t'\|_2,
\end{align*}
where in the last inequality we use that $(1-\lambda_t\eta_t) < 1$ together with the fact that for $\eta_t \leq \frac{2}{L}$ the gradient mapping, $w_t - \eta_t\Gcal(w_t)$, is contractive because $\Gcal(\cdot)$ is $L$-Lipschitz.
We can now apply Theorem~\ref{thm:fedlman_dp} to Algorithm~\ref{alg:sgld}. Further, because $2\geq (2-\lambda_t\eta_t)\geq 1$, we can ignore this factor and so we have
\begin{align*}
\textstyle
    D_{\alpha}(w_T\|w_T') \leq \max_{\tau \in [T]}\frac{2\alpha \eta G^2}{|\Mcal_\tau|^2\beta_0\sum_{t=\tau}^T\lambda_t}.
\end{align*}
For the choice of $|\Mcal_\tau| = \sqrt{\frac{T}{T-\tau+1}}$ and $\lambda_t = \frac{1}{t\eta}$ we have
\begin{align*}
\textstyle
    \max_{\tau \in [T]}\frac{2\alpha \eta G^2}{|\Mcal_\tau|^2\beta_0\sum_{t=\tau}^T\lambda_t}
    = \max_{\tau\in [T]} \frac{2\alpha G^2\eta^2 (T-\tau+1)}{\beta_0 T \sum_{t=\tau}^T\frac{1}{t}} \leq \frac{2\alpha G^2\eta^2}{\beta_0}.
\end{align*}

\end{proof}

\section{Conclusion}

We obtain dimension independent generalization bounds for DP-SGD using convergence analysis of SGLD of \citep{freund2021convergence}. We demonstrate that the derived bounds improve previous results for certain classes of overparameterized smooth convex learning problems. Our results imply that for certain practical applications with overparameterized machine learning models, it is possible to obtain strong differential-privacy guarantees without much degradation in the generalization performance.

\clearpage
\bibliographystyle{plainnat}
\bibliography{myrefs}

\begin{thebibliography}{15}
\providecommand{\natexlab}[1]{#1}
\providecommand{\url}[1]{\texttt{#1}}
\expandafter\ifx\csname urlstyle\endcsname\relax
  \providecommand{\doi}[1]{doi: #1}\else
  \providecommand{\doi}{doi: \begingroup \urlstyle{rm}\Url}\fi

\bibitem[Abadi et~al.(2016)Abadi, Chu, Goodfellow, McMahan, Mironov, Talwar,
  and Zhang]{abadi2016deep}
Martin Abadi, Andy Chu, Ian Goodfellow, H~Brendan McMahan, Ilya Mironov, Kunal
  Talwar, and Li~Zhang.
\newblock Deep learning with differential privacy.
\newblock In \emph{Proceedings of the 2016 ACM SIGSAC conference on computer
  and communications security}, pages 308--318, 2016.

\bibitem[Balle et~al.(2018)Balle, Barthe, and Gaboardi]{Balle_dp_amplification}
Borja Balle, Gilles Barthe, and Marco Gaboardi.
\newblock Privacy amplification by subsampling: Tight analyses via couplings
  and divergences.
\newblock In S.~Bengio, H.~Wallach, H.~Larochelle, K.~Grauman, N.~Cesa-Bianchi,
  and R.~Garnett, editors, \emph{Advances in Neural Information Processing
  Systems}, volume~31, 2018.

\bibitem[Bassily et~al.(2020)Bassily, Feldman, Guzm\'an, and
  Talwar]{Bassily2020}
R.~Bassily, V.~Feldman, C.~Guzm\'an, and K.~Talwar.
\newblock Stability of stochastic gradient descent on nonsmooth convex losses.
\newblock In \emph{Proceedings of 34th Conference on Neural Information
  Processing Systems (Neurips)}, 2020.

\bibitem[Bassily et~al.(2014)Bassily, Smith, and Thakurta]{bassily2014private}
Raef Bassily, Adam Smith, and Abhradeep Thakurta.
\newblock Private empirical risk minimization: Efficient algorithms and tight
  error bounds.
\newblock In \emph{2014 IEEE 55th Annual Symposium on Foundations of Computer
  Science}, pages 464--473. IEEE, 2014.

\bibitem[Bassily et~al.(2019)Bassily, Feldman, Talwar, and
  Thakurta]{bassily2019private}
Raef Bassily, Vitaly Feldman, Kunal Talwar, and Abhradeep Thakurta.
\newblock Private stochastic convex optimization with optimal rates.
\newblock \emph{arXiv preprint arXiv:1908.09970}, 2019.

\bibitem[Dwork et~al.(2010)Dwork, Rothblum, and Vadhan]{dwork2010boosting}
Cynthia Dwork, Guy~N Rothblum, and Salil Vadhan.
\newblock Boosting and differential privacy.
\newblock In \emph{2010 IEEE 51st Annual Symposium on Foundations of Computer
  Science}, pages 51--60. IEEE, 2010.

\bibitem[Dwork et~al.(2014)Dwork, Roth, et~al.]{dwork2014algorithmic}
Cynthia Dwork, Aaron Roth, et~al.
\newblock The algorithmic foundations of differential privacy.
\newblock \emph{Found. Trends Theor. Comput. Sci.}, 9\penalty0 (3-4):\penalty0
  211--407, 2014.

\bibitem[Feldman et~al.(2020)Feldman, Koren, and Talwar]{feldman2020private}
Vitaly Feldman, Tomer Koren, and Kunal Talwar.
\newblock Private stochastic convex optimization: optimal rates in linear time.
\newblock In \emph{Proceedings of the 52nd Annual ACM SIGACT Symposium on
  Theory of Computing}, pages 439--449, 2020.

\bibitem[Freund et~al.(2021)Freund, Ma, and Zhang]{freund2021convergence}
Yoav Freund, Yi-An Ma, and Tong Zhang.
\newblock When is the convergence time of langevin algorithms dimension
  independent? a composite optimization viewpoint.
\newblock \emph{arXiv preprent arXiv:2110.01827}, 2021.

\bibitem[Hardt et~al.(2016)Hardt, Recht, and Singer]{hardt2016train}
Moritz Hardt, Ben Recht, and Yoram Singer.
\newblock Train faster, generalize better: Stability of stochastic gradient
  descent.
\newblock In \emph{International Conference on Machine Learning}, pages
  1225--1234. PMLR, 2016.

\bibitem[Jain et~al.(2019)Jain, Nagaraj, and Netrapalli]{jain2019making}
Prateek Jain, Dheeraj Nagaraj, and Praneeth Netrapalli.
\newblock Making the last iterate of {SGD} information theoretically optimal.
\newblock In \emph{Conference on Learning Theory}, pages 1752--1755. PMLR,
  2019.

\bibitem[Mironov(2017)]{mironov2017renyi}
Ilya Mironov.
\newblock R{\'e}nyi differential privacy.
\newblock In \emph{2017 IEEE 30th Computer Security Foundations Symposium
  (CSF)}, pages 263--275. IEEE, 2017.

\bibitem[Shamir and Zhang(2013)]{shamir2013stochastic}
Ohad Shamir and Tong Zhang.
\newblock Stochastic gradient descent for non-smooth optimization: Convergence
  results and optimal averaging schemes.
\newblock In \emph{International conference on machine learning}, pages 71--79.
  PMLR, 2013.

\bibitem[Song et~al.(2021)Song, Steinke, Thakkar, and
  Thakurta]{song2020evading}
Shuang Song, Thomas Steinke, Om~Thakkar, and Abhradeep Thakurta.
\newblock Evading curse of dimensionality in unconstrained private glms via
  private gradient descent.
\newblock In \emph{AISTAT 2021}, 2021.

\bibitem[Welling and Teh(2011)]{SGLD}
Max Welling and Yee~W Teh.
\newblock Bayesian learning via stochastic gradient {L}angevin dynamics.
\newblock In \emph{Proceedings of the 28th International Conference on Machine
  Learning}, pages 681--688. Omnipress, 2011.

\end{thebibliography}

\end{document}